\documentclass[11pt]{article}

\usepackage[a4paper,margin=25mm]{geometry}
\usepackage[T1]{fontenc}
\usepackage[utf8]{inputenc}
\usepackage{amsmath,amssymb,amsthm}
\usepackage{bm}
\usepackage{authblk}
\usepackage{hyperref}
\usepackage{tikz}
\usetikzlibrary{arrows.meta,positioning,calc}

\hypersetup{colorlinks=true,linkcolor=blue,citecolor=blue,urlcolor=blue}

\title{Spacetime Formation under Requirements:\\
Contextual Realization and Form-Dependent Probability}
\author{Song-Ju Kim\\
SOBIN Institute LLC, Kawanishi, Hyogo, Japan\\
\texttt{kim@sobin.org}}
\date{}

\newtheorem{definition}{Definition}
\newtheorem{postulate}{Postulate}
\newtheorem{principle}{Principle}
\newtheorem{proposition}{Proposition}
\newtheorem{remark}{Remark}

\DeclareMathOperator{\Real}{Real}
\DeclareMathOperator{\Realize}{Realize}
\DeclareMathOperator{\Glue}{Glue}
\DeclareMathOperator{\Inv}{Inv}

\begin{document}
\maketitle

\begin{abstract}
Quantum cognition is usually formulated by placing agents, questions, and outcomes within a fixed temporal and event structure, and then replacing classical probability by quantum probability when order effects, contextuality, or violations of the law of total probability appear. This paper proposes a different interpretation. We argue that quantum probability can be understood as the fixed-spacetime projection of a deeper process: contextual spacetime formation under finite-state requirements. Here ``spacetime'' is used at the transcendental-operational level: the form in which events, objects, order, separation, and possible interaction are realized, not yet a fully specified physical metric manifold.

The framework begins not with time, space, objects, or probabilities, but with requirements: finite representational capacity, single-state semantic stability, context-sensitive intervention, avoidance of explicit context labels, coherent world-formation, and intersubjective transformability. Under these requirements, a fixed global Boolean event structure may fail. The resulting mismatch, when projected back into a classical spacetime form, appears as interference, noncommutativity, and quantum-like probability.

Building on prior single-state approaches to contextuality, especially the information-theoretic obstruction developed in earlier work, we reinterpret classical contextual bookkeeping cost as the fixed-spacetime shadow of contextual spacetime formation. What appears in a classical representation as auxiliary memory or context labels appears, in the proposed framework, as holonomy-like mismatch among locally Boolean logic-worlds. The interference term is then interpreted as the cross term generated when locally classical realization contributions are nontrivially glued and projected back into a fixed
classical spacetime form.

The result is a transcendental-operational realist account of quantum cognition: transcendental because it concerns the conditions under which objecthood, eventhood, probability, and spacetime become possible; operational because contexts act as interventions rather than labels; and realist because objectivity is defined by invariants preserved across observer- and history-dependent spacetime formations.
\end{abstract}

\section{Introduction}

What is given first: time and space, or the requirements under which a world can be realized?

Classical physics, classical probability, and much of cognitive modeling tacitly begin from an already organized world. Events occur in time, objects occupy space, and probabilities are assigned to subsets of a fixed event space. In such a picture, classical probability is natural: events are assumed to belong to a common Boolean algebra, and uncertainty concerns which event in that pre-given space is actual. When a partition \(\{A,\bar A\}\) is available, the law of total probability follows:
\begin{equation}
P(B)=P(A)P(B\mid A)+P(\bar A)P(B\mid \bar A).
\label{eq:intro_ltp}
\end{equation}
The world is first spatially and temporally organized; probability is then placed on top of that organization.

Quantum cognition begins from the observation that human judgment, decision, and categorization often fail to fit this classical picture. Question order effects, conjunction and disjunction fallacies, ambiguity effects, and violations of the law of total probability suggest that the act of asking a question may change the cognitive state in which subsequent questions are answered.
In the quantum-cognition literature, such phenomena are modeled using Hilbert spaces, noncommuting observables, projection operators, and interference terms 
\cite{BusemeyerBruza2012,BlutnerGraben2016,AertsAerts1995,PothosBusemeyer2013,BruzaWangBusemeyer2015,Khrennikov2010}. 
A typical quantum-like expression replaces Eq.~\eqref{eq:intro_ltp} by
\begin{equation}
P_Q(B)=P(A)P(B\mid A)+P(\bar A)P(B\mid \bar A)+I(A,B),
\label{eq:intro_interference}
\end{equation}
where, in the two-context case, the interference term can be written as
\begin{equation}
I(A,B)=2\sqrt{P(A)P(B\mid A)P(\bar A)P(B\mid \bar A)}\cos\theta .
\label{eq:intro_interference_phase}
\end{equation}
This formalism is empirically powerful. Yet its foundational meaning remains unclear. Is quantum probability merely a compact descriptive tool? Is it a psychological heuristic? Or does it reveal a deeper structural feature of finite cognition?

This paper proposes a different answer. We argue that the appearance of quantum probability in cognition is not merely a change in probabilistic calculus. It is the fixed-spacetime shadow of a more basic transformation: the formation of spacetime itself under finite contextual requirements.

The guiding slogan is:
\begin{quote}
There are only requirements. Realization includes time and space.
\end{quote}
By ``requirements'' we mean the constraints under which a finite agent must form a coherent world: limited internal state, single-state semantic stability, context-sensitive response, avoidance of explicit context labels when possible, and intersubjective transformability of experience. On this view, time and space are not simply the a priori containers in which realization occurs. They are themselves part of realization. If one treats time and space as already fixed, then the mismatch among contextual realizations appears as order effects, noncommutativity, and quantum-like interference. But the requirements do not demand a fixed spacetime form. They demand coherent realization under finite resources.
A clarification is essential. By ``spacetime'' this paper does not initially mean a fully specified physical metric manifold, nor does it claim to derive Lorentzian geometry, general relativity, or the quantitative structure of physical spacetime. Rather, spacetime means the form in which requirements are realized as events, objects, temporal order, spatial separation, coexistence, and possible interaction. The claim is therefore not that quantum cognition already provides a complete physical theory of spacetime, but that the very distinction between event, object, order, and separation should be treated as part of contextual realization rather than as an unquestioned background.

This position is inspired by, but also departs from, Kant's transcendental philosophy. For Kant, space and time are not empirical objects discovered in the world; they are forms through which experience of objects becomes possible \cite{Kant1781}. The present paper radicalizes this point in a resource-sensitive and operational direction. Space and time are not treated as once-and-for-all fixed forms. They are allowed to vary with subject, history, and interaction, insofar as different agents or histories realize the requirements of coherent experience differently. Thus, the question is not merely whether cognition occurs in spacetime. The question is how a finite subject forms a spacetime order at all.

Recent prior work on single-state contextuality provides a technical foothold for this move. 
In single-state models of contextuality, an agent or device is required to maintain one semantically stable internal state across multiple contexts. Contexts are not passive labels but interventions acting on the shared state. If a classical model is allowed to store explicit context labels or duplicate internal representations, context-sensitive behavior can be reproduced at the cost of additional memory. But if the internal state must retain fixed semantics across contexts, classical representation faces an irreducible bookkeeping cost \cite{KimQTOW2026,KimObstruction2026,KimNoGo2026}.

The present paper interprets this cost transcendentally. What appears, in a fixed classical representation, as external context bookkeeping may instead be absorbed by transforming the spacetime form itself. The choice is not only between classical probability and quantum probability. It is between two ways of realizing contextual requirements:
\begin{enumerate}
    \item keep a fixed classical spacetime/event form and pay explicit contextual bookkeeping cost;
    \item allow the spacetime/event form to become contextually transformed, in which case the fixed-spacetime projection appears as quantum interference.
\end{enumerate}

The mechanism proposed below is that local logic-worlds contribute locally classical realization weights, but contextual transition maps glue these contributions nontrivially.
When this non-flat gluing is projected into a fixed classical spacetime form, the cross term appears as quantum-like interference.
This motivates the central hypothesis of the paper:
\begin{quote}
Quantum probability is the fixed-spacetime projection of contextual spacetime formation under finite-state requirements.
\end{quote}


Section 2 situates the proposal in relation to quantum cognition, operational realism, single-state obstruction results, and Kantian transcendental philosophy. Section 3 formalizes requirements and requirement-first realization. Section 4 introduces single-state semantics and contexts as interventions. Section 5 defines classical spacetime forms and local logic-worlds. Section 6 develops contextual spacetime formation and interprets interference as projected holonomy-like mismatch. Section 7 draws consequences for objectivity, objects, facts, laws, nonlocality, and generalized incompleteness. Section 8 clarifies the philosophical position, limitations, and future work. Section 9 concludes.

The central claim is not that quantum cognition proves a complete physics of spacetime. Nor is it that the brain must be a microscopic quantum system. The claim is more precise: quantum cognition exposes a structural fault line. When finite subjects integrate multiple contexts under a fixed spacetime/event form, classical probability may fail. This failure suggests that the deeper object of inquiry is not probability alone, but the formation of the spacetime form in which probability, objecthood, and facticity become possible.

\section{Background and Positioning}
\label{sec:background}

\subsection{Quantum Cognition and Violations of Classical Probability}

The usual formulation of quantum cognition begins with a fixed experiential arena: questions are asked in time, answers are registered as events, and probabilities are assigned to those events. When empirical or operational phenomena violate the classical law of total probability, one introduces quantum probability, noncommuting observables, projection operators, or interference terms \cite{BusemeyerBruza2012}.

Blutner and beim Graben argue that quantum cognition should not be understood merely as curve-fitting. Their account connects quantum probability to bounded rationality, operational realism, partial Boolean algebras, and orthomodular structures \cite{BlutnerGraben2016}. 
This foundational orientation is complementary to broader surveys and defenses of quantum probability in cognition, which emphasize order effects, interference, conjunction and disjunction fallacies, and departures from classical Bayesian probability \cite{PothosBusemeyer2013,BruzaWangBusemeyer2015,Khrennikov2010}. 
This is important for the present paper. Quantum cognition becomes foundationally interesting when it is read not as a convenient calculus, but as evidence that cognitive events do not always belong to one global Boolean event algebra.

The present framework accepts this foundational direction but changes the level of explanation. We do not begin with Hilbert space as the primitive formal arena. We begin with requirements: finite state, context-sensitive intervention, coherent world-formation, and intersubjective transformability. Quantum probability is then interpreted as the fixed-spacetime projection of a deeper process by which local logic-worlds are glued into a coherent world.

\subsection{Prior Single-State Obstruction Results}

Prior work on single-state contextuality identifies a representational pressure that is central to this paper. 
Suppose an agent maintains a shared internal state with fixed semantics across contexts. If contextual statistics are to be represented classically, the model must either embed contextual dependence into the internal state or introduce an auxiliary context-bearing memory. In information-theoretic terms, contextuality can be read as a witness of irreducible information cost in classical representations \cite{KimObstruction2026,KimNoGo2026}.

The present paper uses this idea but changes its interpretation.
In that framework, the relevant obstruction concerns context bookkeeping under fixed shared-state semantics. In the present framework, that bookkeeping cost is interpreted as the fixed-spacetime manifestation of a deeper formation cost. Classical representation pays with labels, memory, or context-indexed hidden variables. Contextual spacetime formation pays with noncommutativity, holonomy-like mismatch, and interference.

The central correspondence is therefore:
\begin{equation}
\text{classical contextual memory}\quad \longleftrightarrow \quad\text{contextual spacetime holonomy}.
\label{eq:memory_holonomy_intro}
\end{equation}

\subsection{Kantian Motivation}

Kant's transcendental philosophy holds that space and time are forms of intuition: not empirical objects inside the world, but conditions under which objects can appear. The present framework accepts this transcendental direction but rejects the rigidity of the forms. Space and time are not treated as fixed once and for all. They are dynamically formed under finite-state, context-sensitive requirements.

Thus the present theory can be understood as dynamic or resource-constrained Kantianism:
\begin{quote}
Space and time are forms of coherent realization under requirements.
\end{quote}
This is not a retreat into subjectivism. The subject does not freely invent the world. Rather, requirements constrain possible formations, and objectivity is defined by invariants preserved across transformations among formations.

\section{Requirement-First Realization}
\label{sec:requirements}

\begin{definition}[Requirement Structure]
Let \(\mathcal{R}\) denote a set of requirements imposed on an agent or realizing system. These include, at minimum:
\begin{enumerate}
    \item finite representational capacity;
    \item maintenance of a single semantically stable internal state;
    \item context-sensitive response;
    \item avoidance of explicit context labels when possible;
    \item formation of a coherent world;
    \item intersubjective transformability of experience.
\end{enumerate}
\end{definition}

These requirements are not intended as an arbitrary list. They are the minimal constraints needed for the present problem. Finite representational capacity prevents unlimited duplication of context-indexed states.
Single-state semantic stability requires that the same internal state remain identifiable across contexts.
Context-sensitive response is needed to account for order effects and contextuality.
Avoidance of explicit context labels isolates the nonclassical pressure rather than solving it by external bookkeeping. 
Coherent world-formation requires local outcomes to be integrated into a stable experiential world. 
Intersubjective transformability is required for objectivity rather than private construction.

The crucial assumption is not that the world is subjective, nor that anything can be constructed arbitrarily. The assumption is that what is primitive is not an a priori spacetime container, but a set of constraints or requirements that any realized world must satisfy.

\begin{postulate}[Requirement-First Realization]
Time, space, objects, and probabilities are not assumed as primitive. They arise as components of a realization map
\begin{equation}
\Realize:\mathcal{R}\longrightarrow(\mathfrak{S},\mathcal{E}_{\mathfrak{S}},P_{\mathfrak{S}}),
\label{eq:realize_map}
\end{equation}
where \(\mathfrak{S}\) is a spacetime form, \(\mathcal{E}_{\mathfrak{S}}\) is the event structure generated under that form, and \(P_{\mathfrak{S}}\) is the probability assignment defined on \(\mathcal{E}_{\mathfrak{S}}\).
\end{postulate}

\begin{remark}[Scope of Spacetime]
The term ``spacetime'' is used deliberately. It does not refer here to a merely subjective image or linguistic convention, but to the form through which events, objects, order, separation, and interaction become possible for a finite realizing system. The full metric, causal, and dynamical structure of physical spacetime requires additional assumptions beyond the present framework.
\end{remark}

Operationally, a spacetime form determines which events can be distinguished, which can be treated as co-present alternatives, which must be ordered sequentially, which can be identified across contexts, and which transformations
preserve objecthood and probability. Thus a spacetime form is not an introspective image of space and time, but an operational structure governing event distinction, ordering, separation, co-possibility, and re-identification.

Thus probability is not written simply as \(P\). It is always indexed by a spacetime form:
\begin{equation}
P_{\mathfrak{S}}.
\end{equation}
If the spacetime form changes, the event structure and probability assignment change as well:
\begin{equation}
\mathfrak{S}_{\mathrm{cl}}\rightarrow \mathfrak{S}^{\ast},\qquad
\mathcal{E}_{\mathfrak{S}_{\mathrm{cl}}}\rightarrow \mathcal{E}_{\mathfrak{S}^{\ast}},\qquad
P_{\mathfrak{S}_{\mathrm{cl}}}\rightarrow P_{\mathfrak{S}^{\ast}}.
\label{eq:form_change}
\end{equation}

\begin{remark}
The notation \(\mathfrak{S}\) is reserved for spacetime forms. The internal state of an agent is denoted by \(X\). This distinction is essential: the theory does not identify the agent's internal state with spacetime itself. It claims that spacetime form is part of the realization through which requirements become a coherent world.
\end{remark}

\section{Single-State Semantics and Contextual Intervention}
\label{sec:single_state}

We now introduce the finite-agent constraint, aligned with single-state contextuality theories.

\begin{postulate}[Single-State Semantics]
An agent maintains an internal state \(X\) whose meaning does not fragment into context-indexed sub-states. Before contextual intervention, the internal state does not explicitly encode the context:
\begin{equation}
I(X;C)=0,
\label{eq:single_state_semantics}
\end{equation}
where \(C\) denotes the context variable.
\end{postulate}

Equation~\eqref{eq:single_state_semantics} does not mean that behavior is noncontextual. It means that context is not stored as an explicit internal label before the intervention. Context sensitivity must therefore be realized through operations acting on the shared state, not through context-indexed copies of the state.

\begin{definition}[Context as Intervention]
A context \(c\in C\) is not a passive label. It is an intervention
\begin{equation}
T_c:X\longrightarrow X_c,
\label{eq:context_intervention}
\end{equation}
which may change subsequent outcome statistics.
\end{definition}

Two contexts \(a,b\in C\) are noncommuting when
\begin{equation}
T_aT_b \neq T_bT_a.
\label{eq:noncommuting_contexts}
\end{equation}
This noncommutativity is a minimal seed of operational temporal ordering. If applying \(a\) before \(b\) differs from applying \(b\) before \(a\), then the agent cannot treat the two contexts as merely co-present Boolean predicates. They must be serialized, ordered, or otherwise separated in the realized world.

\subsection{Classical Context Bookkeeping}

If a classical representation insists on a fixed event structure while reproducing contextual
statistics, it must often introduce an auxiliary memory or label \(M\) that carries contextual information.
This issue is also close to context--content approaches to contextuality, in which random
variables are indexed not only by their contents but also by their contexts
\cite{DzhafarovKujala2016}. 
The single-state condition \(I(X;C)=0\) concerns the pre-intervention shared state. The bookkeeping cost concerns an additional representation \(M\), not the shared state \(X\) itself.

A conservative schematic form of the cost is therefore:
\begin{equation}
I(M;C)>0,
\label{eq:memory_context_info}
\end{equation}
or, in an ontological setting,
\begin{equation}
  H(M) \geq I(C;O\mid \lambda)
\quad \text{schematically},
\label{eq:kim_cost_ontic_clean}
\end{equation}
where \(O\) denotes outcomes and \(\lambda\) an ontic state.
In contextual regimes, this conditional mutual information is positive.
The inequality is schematic because the exact bound depends on the modeling assumptions and coding conventions.
The important point is structural: a classical fixed-form representation pays a nonzero contextual bookkeeping cost.

\section{Classical Spacetime Form and Local Logic-Worlds}
\label{sec:logic_worlds}

A classical spacetime form \(\mathfrak{S}_{\mathrm{cl}}\) is defined by the assumption that all relevant events can be embedded into a single global Boolean event algebra \(\Sigma\). Under this assumption, for any event \(B\) and contextual partition \(\{A,\bar A\}\), the law of total probability holds:
\begin{equation}
P_{\mathfrak{S}_{\mathrm{cl}}}(B)=P(A)P(B\mid A)+P(\bar A)P(B\mid \bar A).
\label{eq:ltp_spacetime_indexed}
\end{equation}
This is not merely a probability law. It expresses a spacetime-form assumption: all relevant alternatives are jointly placeable within a single global logic-world.

\begin{definition}[Logic-World]
A logic-world \(L_c\) is a local event structure generated under a context \(c\). Within a single context, \(L_c\) may be Boolean. The global world is not assumed to be a primitive Boolean whole, but a gluing of local logic-worlds:
\begin{equation}
\mathcal{W}=\Glue\bigl(\{L_c\}_{c\in C}\bigr).
\label{eq:world_gluing}
\end{equation}
\end{definition}

Classical cognition assumes that this gluing is globally Boolean:
\begin{equation}
\Glue\bigl(\{L_c\}_{c\in C}\bigr)=\Sigma.
\label{eq:flat_gluing}
\end{equation}
Quantum-like cognition arises when no such global Boolean gluing is available.
This use of local event structures and global gluing is conceptually related to sheaf-theoretic approaches to contextuality, where contextuality and nonlocality are characterized as obstructions to the existence of global sections
\cite{AbramskyBrandenburger2011}. The present framework does not adopt the sheaf-theoretic formalism directly, but it uses the same structural contrast between local compatibility and global closure.

\begin{proposition}[Classical Closure Failure]
If local logic-worlds \(L_a,L_b\) are generated by noncommuting contextual interventions, then in general there is no single global Boolean event algebra \(\Sigma\) preserving all local probability assignments and sequential statistics.
\end{proposition}

\begin{proof}[Sketch]
If \(T_aT_b\neq T_bT_a\), then the outcome statistics of \(b\) may depend on whether \(a\) has been applied before it. Thus \(A\cap B\) cannot be treated as a context-free Boolean intersection independent of order. A global Boolean algebra requires context-independent logical operations and a joint probability assignment over all relevant events. Noncommuting intervention statistics obstruct this embedding. Hence the failure is not merely probabilistic but structural: the global event space assumed by the classical spacetime form is not available.
\end{proof}

\section{Contextual Spacetime Formation}
\label{sec:spacetime_formation}

We now introduce the central theoretical move.

\begin{definition}[Contextual Spacetime Form]
A contextual spacetime form \(\mathfrak{S}^{\ast}\) is a realization of requirements in which local logic-worlds are not forced into a single global Boolean event algebra. Instead, they are glued by transition maps
\begin{equation}
G_{ab}:L_a\longrightarrow L_b
\label{eq:transition_maps}
\end{equation}
that may carry order, phase, projection, or holonomy-like mismatch.
\end{definition}

Here gluing is non-flat when contextual transition maps fail to compose trivially around a contextual loop. For example, for local logic-worlds \(L_a\) and \(L_b\),
\[
G_{ba}G_{ab}\not\simeq \mathrm{id}_{L_a}.
\]
The residual failure of this closed transport to return identically is what we call a holonomy-like mismatch.

\paragraph{Mechanism: From Gluing Mismatch to Interference.}
The crucial point is that the interference term is not introduced as an additional
probability rule. It arises as a cross term when locally classical realization
contributions are glued nontrivially and then projected into a fixed classical
spacetime form.

Consider two local logic-worlds associated with the alternatives \(A\) and
\(\bar A\), and an event \(B\) to be realized. Locally, each branch contributes a
classical realization weight:
\[
r_A=\sqrt{P(A)P(B\mid A)},\qquad
r_{\bar A}=\sqrt{P(\bar A)P(B\mid \bar A)}.
\]
In a flat classical gluing, these alternatives belong to one global Boolean
event algebra. Their contributions are mutually exclusive and therefore add as
probabilities:
\[
P_{\mathfrak S_{\mathrm{cl}}}(B)=r_A^2+r_{\bar A}^2.
\]


In a contextual spacetime form, however, the two local logic-worlds are not assumed to be globally and flatly glued. Their contributions must first be transported through contextual transition maps. This transport may introduce a
relative mismatch of orientation, phase, or compatibility.
When contextual transport preserves the magnitude of local realization weights while changing only their compatibility relation, the minimal two-branch representation is phase-like.
We represent the transported realization contributions as
\[
\rho_A=r_A e^{i\phi_A},\qquad
\rho_{\bar A}=r_{\bar A}e^{i\phi_{\bar A}}.
\]
The relative gluing mismatch is then
\[
\Theta_{\mathfrak S^\ast}(A,B)=\phi_{\bar A}-\phi_A.
\]
This does not imply that all contextual updates are unitary. It only states that the interference-relevant residue of non-flat gluing can be represented, in the two-branch case, by a norm-preserving relative phase.

Thus \(\Theta_{\mathfrak S^\ast}(A,B)\) is not a free psychological fitting parameter. It represents the residual mismatch produced by transporting and gluing the local logic-worlds associated with \(A\), \(\bar A\), and \(B\).
When the contextual spacetime form is displayed inside the fixed classical spacetime form, the realized probability is the squared norm of the glued realization contribution:
\[
P_{\mathfrak S^\ast\to\mathfrak S_{\mathrm{cl}}}(B)
=
\left|\rho_A+\rho_{\bar A}\right|^2.
\]
Expanding this expression gives
\[
P_{\mathfrak S^\ast\to\mathfrak S_{\mathrm{cl}}}(B)
=
P(A)P(B\mid A)
+
P(\bar A)P(B\mid \bar A)
+
2\sqrt{P(A)P(B\mid A)P(\bar A)P(B\mid \bar A)}
\cos\Theta_{\mathfrak S^\ast}(A,B).
\]
The usual quantum-like interference term is therefore not postulated from outside. It is the fixed-spacetime projection of a non-flat gluing of locally classical realization contributions.

\begin{figure}[t]
\centering
\begin{tikzpicture}[
    node distance=1.6cm and 2.4cm,
    box/.style={draw, rounded corners, align=center, minimum width=3.0cm, minimum height=0.9cm},
    smallbox/.style={draw, rounded corners, align=center, minimum width=2.6cm, minimum height=0.8cm},
    arrow/.style={-{Latex[length=2mm]}, thick}
]

\node[box] (LA) {$L_A$\\local logic-world};
\node[box, right=of LA] (LB) {$L_{\bar A}$\\local logic-world};

\node[smallbox, below=of LA] (rA) {$r_A=\sqrt{P(A)P(B\mid A)}$};
\node[smallbox, below=of LB] (rB) {$r_{\bar A}=\sqrt{P(\bar A)P(B\mid \bar A)}$};

\coordinate (midr) at ($(rA)!0.5!(rB)$);

\node[box, below=1.3cm of midr] (glue)
{non-flat gluing\\
contextual transport\\
\(\Theta_{\mathfrak S^\ast}(A,B)\)};

\node[box, below=of glue] (proj)
{projection into fixed\\classical spacetime form\\
\(\mathfrak S^\ast\to\mathfrak S_{\mathrm{cl}}\)};

\node[box, below=of proj] (interf)
{interference cross term\\
\(2r_A r_{\bar A}\cos\Theta_{\mathfrak S^\ast}(A,B)\)};

\draw[arrow] (LA) -- (rA);
\draw[arrow] (LB) -- (rB);
\draw[arrow] (rA) -- (glue);
\draw[arrow] (rB) -- (glue);
\draw[arrow] (glue) -- (proj);
\draw[arrow] (proj) -- (interf);

\draw[arrow, dashed] (LA) to[bend left=15] node[above] {\(G_{A\bar A}\)} (LB);
\draw[arrow, dashed] (LB) to[bend left=15] node[below] {\(G_{\bar A A}\)} (LA);

\end{tikzpicture}
\caption{
Schematic mechanism of contextual spacetime formation. Local logic-worlds contribute
locally classical realization weights. When these contributions are transported and
nontrivially glued through contextual transition maps, a holonomy-like mismatch
\(\Theta_{\mathfrak S^\ast}(A,B)\) is generated. Projection into a fixed classical
spacetime form produces the interference cross term.
}
\label{fig:gluing-mechanism}
\end{figure}
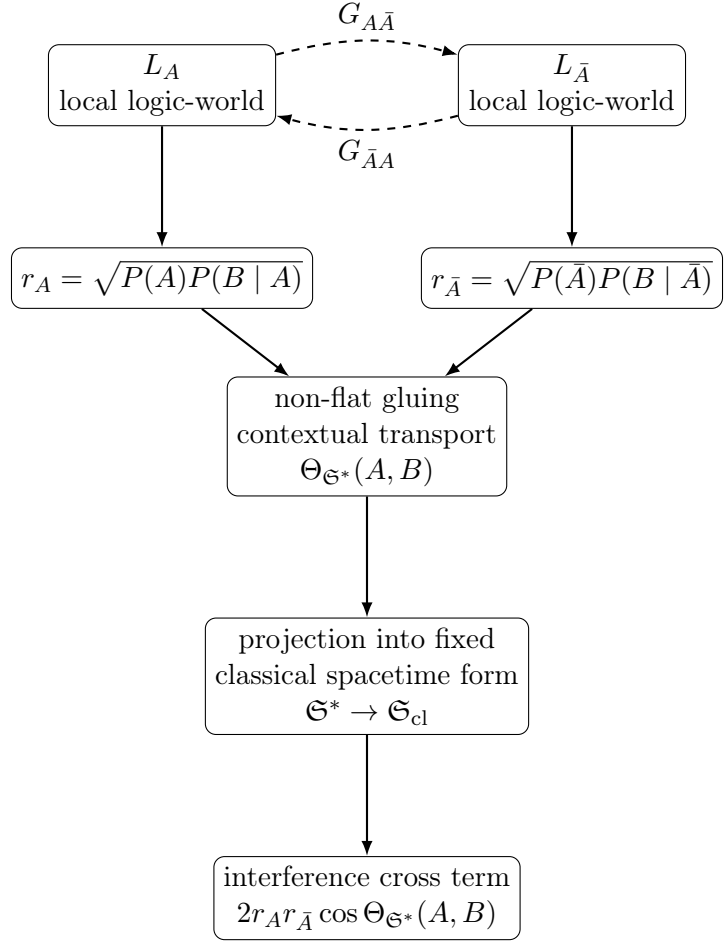

Figure~\ref{fig:gluing-mechanism} summarizes the mechanism: the interference term is the projected cross term of non-flat contextual gluing, not an independently postulated correction to classical probability.



More generally, the resulting probability assignment is not the old \(P\) on the old event space. It is
\begin{equation}
P_{\mathfrak{S}^{\ast}}:\mathcal{E}_{\mathfrak{S}^{\ast}}\longrightarrow [0,1],
\end{equation}
where \(\mathcal{E}_{\mathfrak{S}^{\ast}}\) is the event structure produced by contextual spacetime formation.

The preceding mechanism motivates the following general projection form:
\[
P_{\mathfrak S^\ast\to \mathfrak S_{\mathrm{cl}}}(B)
=
P(A)P(B\mid A)+P(\bar A)P(B\mid \bar A)
+
I_{\mathfrak S^\ast}(A,B).
\]
Here the arrow emphasizes that the expression is the projection of a contextual
spacetime probability into a classical display form. In the two-branch case,
the projected mismatch can be parameterized as
\[
I_{\mathfrak S^\ast}(A,B)
=
2\sqrt{P(A)P(B\mid A)P(\bar A)P(B\mid \bar A)}
\cos\Theta_{\mathfrak S^\ast}(A,B).
\]

\begin{definition}[Interference as Projected Spacetime Holonomy]
The interference term \(I_{\mathfrak{S}^{\ast}}(A,B)\) is the fixed-classical-spacetime projection of the nontrivial gluing between local logic-worlds associated with \(A\) and \(B\). It measures the failure of contextual spacetime formation to be globally reducible to \(\mathfrak{S}_{\mathrm{cl}}\).
\end{definition}

In the two-branch case, this parameterization is exactly the form derived in the mechanism above.
%

\begin{remark}[Status of Holonomy]
At this stage, the term ``holonomy'' is used in an operational-structural
sense. The parameter \(\Theta_{\mathfrak S^\ast}(A,B)\) denotes the residual
mismatch generated by contextual transport and gluing among local logic-worlds.
A fully geometric theory would require specifying the relevant space of contexts,
transition maps, and connection-like structure from which \(\Theta_{\mathfrak S^\ast}\)
is derived. The present paper identifies this target but does not yet complete
that derivation.
\end{remark}

\begin{principle}[Projection Principle]
Quantum probability is what contextual spacetime formation looks like when projected onto an a priori fixed classical spacetime form.
\end{principle}

\subsection{Bookkeeping--Geometry Conversion}

The relation to prior single-state obstruction results can now be expressed cleanly.
Under fixed shared-state semantics, a classical model that reproduces contextual statistics must pay for context dependence through labels, memory, or context-indexed hidden variables. The present theory accepts this obstruction but shifts its location.

\begin{principle}[Bookkeeping--Geometry Conversion]
Classical contextual bookkeeping cost can be converted into contextual spacetime formation. What appears in a fixed classical representation as an external label or memory cost appears in contextual spacetime formation as nontrivial gluing, represented in projection by holonomy-like phase:
\begin{equation}
I(M;C)>0\quad \longleftrightarrow \quad\Theta_{\mathfrak{S}^{\ast}}\neq 0.
\label{eq:bookkeeping_geometry}
\end{equation}
\end{principle}

More precisely, there are two strategies for satisfying the requirements:
\begin{enumerate}
    \item \textbf{Classical bookkeeping strategy:} keep a fixed spacetime form \(\mathfrak{S}_{\mathrm{cl}}\), and pay additional context-label memory cost \(M\).
    \item \textbf{Contextual spacetime strategy:} avoid explicit context labels, maintain single-state semantics, and allow the spacetime form to become nontrivially glued, producing \(I_{\mathfrak{S}^{\ast}}\neq 0\).
\end{enumerate}

In this sense, quantum probability is not opposed to classical probability merely as one calculus to another. It is an alternative realization strategy under finite-state requirements.

This conversion also yields an empirical implication. If a task is modified so that subjects can explicitly store, rehearse, or label the relevant context, then part of the contextual burden is shifted from spacetime gluing to explicit memory.
The theory therefore predicts a reduction of order effects and interference-like deviations when reliable external context memory is provided:
\begin{equation}
I(M; C)\uparrow \quad \Rightarrow \quad
|I_{\mathfrak S^\ast}(A,B)|\downarrow .
\label{eq:memory-interference-tradeoff}
\end{equation}
This prediction, summarized in Eq.~\eqref{eq:memory-interference-tradeoff}, distinguishes the present interpretation from a purely descriptive quantum-probability fit.
Empirical quantum-cognition work on question order effects provides a natural testing ground for this prediction, since such effects have been modeled by noncommuting question operators and tested against precise quantum-probability constraints \cite{WangBusemeyer2013}.
In this light, question-order effects can be reinterpreted as cases in which the second question is not answered within the same flat event structure as the first.
The order manipulation changes the contextual transition through which the later answer is realized. The present theory predicts that if the relevant previous context is externally stabilized---for example by displaying prior questions and answers, by allowing rehearsal, or by requiring explicit context labels---then the fitted interference term should decrease. Existing order-effect paradigms therefore provide a natural starting point for testing the bookkeeping--geometry trade-off.

\begin{principle}[Requirement-First Contextual Spacetime Formation]
Let \(\mathcal{R}\) be the requirements of finite-state coherent world-formation: single-state semantics, context-sensitive response, avoidance of explicit context labels, and intersubjective transformability. If these requirements cannot be realized within a single global Boolean spacetime form \(\mathfrak{S}_{\mathrm{cl}}\), then the realizing system must either:
\begin{enumerate}
    \item introduce external contextual bookkeeping \(M\), or
    \item transform the spacetime form itself into a contextual form \(\mathfrak{S}^{\ast}\).
\end{enumerate}
When the second route is taken, the fixed-spacetime projection of \(P_{\mathfrak{S}^{\ast}}\) has the form of quantum probability with interference.
\end{principle}

\section{Consequences}
\label{sec:consequences}

The previous sections introduced the central theoretical move of the paper: quantum-like probability is interpreted not as a primitive modification of probability, but as the fixed-spacetime projection of contextual spacetime formation under finite-state requirements. We now draw several immediate consequences. These are structural consequences, not claims that the present framework already derives the full quantitative content of quantum mechanics, general relativity, or Gödel's incompleteness theorems.

\subsection{Observer- and History-Dependent Spacetime Forms}

If spacetime is not primitive, then it need not be identical for every realizing subject. Let \(i\) denote a subject or observer, and let \(h_i\) denote its history of contextual interventions. The spacetime form realized by that subject is written
\begin{equation}
\mathfrak{S}_i(h_i).
\label{eq:observer_spacetime}
\end{equation}
The associated event structure and probability assignment are
\begin{equation}
\mathcal{E}_{\mathfrak{S}_i(h_i)},\qquad P_{\mathfrak{S}_i(h_i)}.
\end{equation}

This notation expresses a central departure from both Newtonian and standard Kantian assumptions. In Newtonian thought, time and space are absolute containers. In Kantian thought, time and space are a priori forms of human intuition. In the present framework, they are neither absolute containers nor fixed once-and-for-all forms. They are dynamically realized forms that depend on the finite subject, its history, and its mode of interaction.

This does not imply that each subject lives in an arbitrary private universe. The claim is weaker and more precise:
\begin{quote}
Different subjects and histories may realize different spacetime forms, but objectivity consists in the invariant structures preserved across transformations among those forms.
\end{quote}
Let \(T_{ij}\) denote a transformation relating the spacetime form of subject \(i\) to that of subject \(j\):
\begin{equation}
T_{ij}:\mathfrak{S}_i(h_i)\longrightarrow \mathfrak{S}_j(h_j).
\end{equation}
Then an objective structure is not a structure embedded in an absolute spacetime, but a structure preserved under such transformations:
\begin{equation}
\mathcal{O}=\Inv\left(\{\mathfrak{S}_i(h_i),P_{\mathfrak{S}_i(h_i)}\}_{i,h}\right).
\label{eq:objectivity_invariant}
\end{equation}

\begin{definition}[Objectivity]
Objectivity is invariance under transformations among subject- and history-dependent spacetime formations.
\end{definition}

This replaces absolute-spacetime objectivity by transformation-invariant objectivity. The world remains real, but its reality is not defined by membership in one pre-given spacetime container. It is defined by constraints and invariants that survive multiple realizations.

\subsection{Physical Objects, Facts, and Laws as Invariants}

A physical object is not introduced as a primitive occupant of space. Rather, it is defined as a stable equivalence class across multiple spacetime formations.

Let \(A_i\in\mathcal{E}_{\mathfrak{S}_i(h_i)}\) be an event or object-like structure appearing for subject \(i\). Let \(A_j\in\mathcal{E}_{\mathfrak{S}_j(h_j)}\) be the corresponding structure appearing for subject \(j\). We say that these are realizations of the same physical object when
\begin{equation}
T_{ij}(A_i)\sim A_j
\end{equation}
under the relevant transformations and operational tests. Thus:
\begin{equation}
\mathrm{Object}=[A_i]_{i,h,T}.
\label{eq:object_equivalence}
\end{equation}

A fact is then a high-stability event under intersubjective calibration:
\begin{equation}
\mathrm{Fact}(A)\quad \Longleftrightarrow \quad P_{\mathfrak{S}_i(h_i)}(A_i)\approx 1\quad \text{for many } i,h_i,
\label{eq:fact_stability}
\end{equation}
with compatibility under transformations \(T_{ij}\).

A physical law is not primarily a rule imposed inside an absolute spacetime. It is a relation preserved across transformations among spacetime forms. Let \(L_i\) be a lawful relation expressed in \(\mathfrak{S}_i(h_i)\). A law is objective if
\begin{equation}
L_i(A_i,B_i,\ldots)=L_j(T_{ij}A_i,T_{ij}B_i,\ldots)
\label{eq:law_invariance}
\end{equation}
for the relevant transformations \(T_{ij}\).

\begin{definition}[Law]
A physical law is a transformation-invariant relation among event structures across subject-, history-, and interaction-dependent spacetime formations.
\end{definition}

This formulation preserves scientific objectivity while removing the need for an absolute spacetime background. Science becomes the practice of identifying, stabilizing, and formalizing invariants across many forms of realization.

\subsection{Nonlocality as the Projection of Pre-Spatial Relation}

Quantum nonlocality is usually formulated against the background of a pre-given spacetime. Two systems are assumed to be spatially separated. Measurements are performed at distant locations. The problem is then to explain correlations that cannot be decomposed into local hidden-variable form:
\begin{equation}
P_{\mathrm{cl}}(a,b\mid x,y)=\sum_{\lambda}\mu(\lambda)P(a\mid x,\lambda)P(b\mid y,\lambda).
\label{eq:bell_local}
\end{equation}
Here \(x\) and \(y\) are local measurement settings, \(a\) and \(b\) are outcomes, and \(\lambda\) is a shared hidden variable. Bell-type violations show the failure of such a decomposition under the relevant assumptions \cite{Bell1964}.

The present framework changes the order of explanation. It does not begin from two already localized systems in an already given spacetime. It begins from a pre-spatial relation structure
\begin{equation}
R_{AB},
\end{equation}
which is realized as spatially separated events only after contextual spacetime formation. Thus the quantum joint probability is written not as a local decomposition over pre-existing spatial parts, but as
\begin{equation}
P_{\mathfrak{S}^{\ast}}(a,b\mid x,y)=\Phi_{\mathfrak{S}^{\ast}}(R_{AB};x,y),
\label{eq:nonlocal_relation}
\end{equation}
where \(\Phi_{\mathfrak{S}^{\ast}}\) is the realization map induced by contextual spacetime formation.

\begin{principle}[Nonlocality as Projection]
Quantum nonlocality is the appearance, within a formed spacetime, of a relation that is not originally decomposed into local spatial parts.
\end{principle}

In this view, nonlocality is not action across a pre-given space. It is the fixed-spacetime projection of a relation whose unity is prior to spatial separation. The no-signalling condition remains essential:
\begin{align}
\sum_b P_{\mathfrak{S}^{\ast}}(a,b\mid x,y) &= P_{\mathfrak{S}^{\ast}}(a\mid x),\\
\sum_a P_{\mathfrak{S}^{\ast}}(a,b\mid x,y) &= P_{\mathfrak{S}^{\ast}}(b\mid y).
\end{align}
This means that the correlation cannot be used as a signal propagating through space. The reason, in the present interpretation, is not that an influence travels instantaneously but harmlessly. Rather, the correlation does not arise from propagation inside space in the first place. It arises from the projection of a pre-spatial relation into a spacetime form in which its realized components appear separated.

This interpretation does not by itself derive the quantitative Tsirelson bound or the full structure of Bell correlations. Its role is interpretive and structural:
\begin{quote}
Bell-type nonlocality shows that the assumption of a globally local Boolean spacetime decomposition is too strong.
\end{quote}
Within the present theory, this is the spatial analogue of contextuality. Contextuality shows that a family of contextual operations cannot always be embedded into one global Boolean event structure. Nonlocality shows that a relational structure cannot always be decomposed into independent local parts inside one pre-given spacetime.

\subsection{Gödelian Incompleteness and Closure Failure}

The same requirement-first framework also suggests a connection with Gödelian incompleteness. The claim is not that quantum cognition proves Gödel's theorem. Gödel's incompleteness theorems remain mathematical results about formal systems satisfying precise conditions such as consistency, effective axiomatizability, and sufficient arithmetic strength \cite{Godel1931}. The present claim is structural:
\begin{quote}
Gödelian incompleteness can be viewed as a paradigmatic logical instance of a broader closure-failure schema.
\end{quote}

Let \(\mathcal{R}\) be the set of requirements to be realized, and let \(L\) be a finite logic-world. Define
\begin{equation}
\Real_L(\mathcal{R})
\end{equation}
as the subset of requirements realizable within \(L\). A finite logic-world is complete with respect to \(\mathcal{R}\) only if
\begin{equation}
\Real_L(\mathcal{R})=\mathcal{R}.
\end{equation}
In general, however,
\begin{equation}
\Real_L(\mathcal{R})\subsetneq \mathcal{R}.
\label{eq:realization_incomplete}
\end{equation}
The residual is
\begin{equation}
\Delta_L=\mathcal{R}\setminus \Real_L(\mathcal{R}).
\label{eq:delta_logic}
\end{equation}

The same residual structure appears differently in different domains:
\begin{equation}
\Delta_L\sim
\begin{cases}
\text{undecidable propositions}, & \text{formal logic},\\
\text{contextuality and interference}, & \text{probability},\\
\text{nonlocal relational constraints}, & \text{quantum physics},\\
\text{spacetime transformation}, & \text{experience}.
\end{cases}
\label{eq:delta_domains}
\end{equation}

\begin{principle}[Generalized Incompleteness]
No finite realization system can close all requirements within a single internal logic-world. What cannot be closed appears as meta-level information, contextuality, interference, spacetime transformation, or undecidability.
\end{principle}

The connection to Gödel is therefore neither a derivation nor an identity. It is a controlled structural analogy. In each case, a finite formal or experiential system fails to close all requirements within a single internal logic-world.

In compressed form:
\begin{quote}
Gödel is incompleteness in logic-worlds. Quantum contextuality is incompleteness in probability-worlds. Spacetime formation is the operation that re-realizes this incompleteness as a coherent world.
\end{quote}

\section{Discussion}
\label{sec:discussion}

\subsection{From Quantum Cognition to Transcendental Structure}

Quantum cognition has often been understood as a mathematical framework for modeling order effects, conjunction and disjunction fallacies, ambiguity, and violations of classical probability. The present framework shifts the emphasis. Quantum cognition is treated not merely as a model of judgment, but as a diagnostic of the conditions under which a coherent world can be formed by a finite agent.

The key step is to replace the question
\begin{quote}
Why does cognition use quantum probability?
\end{quote}
with the deeper question
\begin{quote}
What kind of spacetime form must a finite subject construct in order to integrate multiple contexts without duplicating internal representations?
\end{quote}
Under this reinterpretation, quantum probability is not the starting point. It is the fixed-spacetime projection of a more primitive process: contextual realization under requirements.

\subsection{Relation to Prior Single-State Obstruction Results}

Prior work on single-state contextuality identifies a precise representational pressure.
If an agent or system maintains a single internal state with fixed semantics across contexts, then contextual dependence cannot always be represented within a single classical probability space without additional contextual information. In a classical embedding, this additional information appears as an auxiliary label or memory \(M\).

The present theory accepts this obstruction but changes its interpretation. What appears as contextual bookkeeping cost in a fixed classical representation is interpreted as spacetime-formation cost at the transcendental level. A classical system pays with labels, memory, or context-indexed hidden variables. A contextual spacetime formation pays with holonomy-like mismatch, noncommutativity, and interference.

Thus:
\begin{equation}
\text{classical contextual memory}\quad \longleftrightarrow \quad\text{contextual spacetime holonomy}.
\label{eq:classical_memory_holonomy}
\end{equation}
This is the main bridge between prior single-state obstruction results and the present requirement-first theory.

\subsection{Why This Is Not Subjective Idealism}

The framework developed here may appear, at first glance, to imply that the world is constructed by the subject. This would be a misunderstanding. The claim is not that the subject freely creates reality. The subject is constrained by requirements, interventions, resistance, and intersubjective invariants.

A subject may form a spacetime
\begin{equation}
\mathfrak{S}_i(h_i),
\end{equation}
depending on its structure and history. But objective reality is not identified with any single \(\mathfrak{S}_i(h_i)\). It is identified with the invariant structure preserved across many such formations:
\begin{equation}
\mathrm{Reality}=\Inv\left(\{\mathfrak{S}_i(h_i),P_{\mathfrak{S}_i(h_i)}\}_{i,h}\right).
\label{eq:reality_inv}
\end{equation}

Therefore, the loss of absolute spacetime does not imply the loss of reality. It implies a transformation of the meaning of reality: reality is no longer what occupies an absolute container, but what constrains and survives transformations among contextual spacetime formations.

\subsection{Philosophical Position: Transcendental-Operational Realism}

The philosophical position of this paper is best described as \emph{transcendental-operational realism}.

It is \emph{transcendental} because it asks for the conditions under which objects, facts, probability, and spacetime can appear at all. In this respect, it inherits the Kantian insight that space and time are not empirical objects inside the world, but forms through which a world becomes experienceable.

It is \emph{operational} because contexts are not passive labels. They are interventions that modify internal states, outcome statistics, and the structure of possible experience.

It is \emph{realist} because the subject does not determine the world arbitrarily. Reality is what remains invariant, resistant, and reproducible across transformations among subjects, histories, and operations.

The framework therefore differs from three familiar positions:
\begin{enumerate}
    \item It differs from classical materialism because it does not assume objects in pre-given spacetime as primitive.
    \item It differs from subjective idealism because it does not reduce reality to private experience.
    \item It differs from instrumentalism because it treats invariants across operations as genuinely constraining structures, not merely as useful descriptions.
\end{enumerate}

\subsection{Scientific Objectivity Without Absolute Spacetime}

A possible concern is that the loss of absolute spacetime undermines scientific objectivity. The present theory argues the opposite. Scientific objectivity becomes more precisely defined.

Scientific measurement does not remove the subject by accessing a view from nowhere. Rather, it stabilizes transformations among subjects, instruments, histories, and contexts. Instruments externalize memory, standardize operations, and reduce private variation in spacetime formation. They do not bypass realization; they constrain it.

A scientific fact is therefore an event structure that remains stable under many forms of calibration:
\begin{equation}
P_{\mathfrak{S}_i(h_i)}(A_i)\approx 1\quad\text{for many } i,h_i,
\end{equation}
with compatibility under transformations \(T_{ij}\). A physical law is a relation preserved under such transformations:
\begin{equation}
L_i(A_i,B_i,\ldots)=L_j(T_{ij}A_i,T_{ij}B_i,\ldots).
\end{equation}
Thus, objectivity is not abolished. It is redefined as transformation-invariance across realizations.

This view is also compatible with recent work on open computation and quantum-like coherence in chemical reaction systems, where coherence is not imposed inside a closed representational space but emerges from the interaction between a process and its environment \cite{Gunji2026}. Such work suggests a possible material analogue of the present transcendental-operational claim: objectivity and coherence arise not by bypassing realization, but by stabilizing interactions between system and environment.

\subsection{Limits of the Present Claim}

The present framework is intentionally foundational and programmatic. Several limitations should be stated clearly.

First, the framework does not derive the full quantitative structure of quantum mechanics. It provides an interpretation of why quantum-like probability arises under finite contextual requirements, but it does not yet derive the Hilbert-space formalism, Born rule, or Tsirelson bound from first principles.

Second, the framework does not claim that human cognition physically generates the universe, nor that the present theory already derives physical spacetime in the sense of Lorentzian geometry, general relativity, or quantum gravity. It claims that spacetime, eventhood, probability, and objecthood must be treated as parts of requirement-realization rather than as primitive givens. The term ``spacetime'' is therefore used at the transcendental-operational level: it names the form in which events, objects, temporal order, spatial separation, and possible interaction are realized. Any extension from this level to a full physical theory of spacetime requires additional mathematical, empirical, and dynamical constraints.

Third, the framework does not claim that human brains are microscopic quantum systems. Quantum probability is treated as an effective representational structure arising from contextual requirements, not as evidence for quantum coherence in neural substrates.

Fourth, the connection to Gödel is structural rather than deductive. The present theory does not derive Gödel's incompleteness theorem. It places Gödelian incompleteness within a broader family of closure failures.

Finally, the notion of contextual spacetime holonomy remains to be developed quantitatively. The phase \(\Theta_{\mathfrak{S}^{\ast}}(A,B)\) should ultimately be derived from properties of contextual interventions and transition maps, not merely fitted to data.

\subsection{Future Work}

Several directions follow naturally.

First, one should develop explicit models of the transformation
\begin{equation}
\mathfrak{S}_{\mathrm{cl}}\rightarrow \mathfrak{S}^{\ast}
\end{equation}
and specify how \(\Theta_{\mathfrak{S}^{\ast}}(A,B)\) is generated by noncommuting contextual operations.

Second, empirical quantum-cognition data should be reanalyzed by treating interference terms as estimates of contextual spacetime holonomy rather than as free phase parameters.

Third, connections to information geometry, category theory, sheaf-theoretic contextuality, and operational reconstructions of quantum theory may provide a more precise mathematical language for logic-world gluing.

Fourth, cognitive experiments should test the trade-off predicted in Section 6.1 between explicit context memory and interference-like deviations. If additional memory, rehearsal, or explicit context labels reduce order effects, this would support the interpretation of quantum-like effects as a response to finite contextual integration.

Fifth, the present framework leaves open the dimensional structure of spacetime.
It does not yet explain why realized physical spacetime has three spatial dimensions and one temporal dimension. A future theory would need to derive dimensionality, temporal direction, and metric-causal structure from additional constraints on requirement-realization, information packing, and intersubjective invariance.

Related quantum-like models of perceptual dynamics, such as quantum-Zeno-type accounts of bistable perception, may also provide useful test cases for the relation between contextual intervention, temporal ordering, and stabilization
\cite{AtmanspacherFilkRomer2004}.

A further empirical direction is to relate the requirements used here to measurable cognitive resources such as attention, working memory, rehearsal, and external memory support.

A concrete experimental test would combine question-order paradigms with working-memory manipulations. For example, an \(N\)-back or digit-span load could be used to reduce the availability of internal contextual bookkeeping, while an external-support condition displays previous questions, answers, or context labels. The theory predicts that increased internal memory load should amplify order effects and fitted interference terms, whereas reliable external context support should reduce them. A crossed design would therefore test whether interference-like deviations track the availability of contextual bookkeeping resources rather than merely reflecting a fixed descriptive quantum parameter.

\section{Conclusion}
\label{sec:conclusion}

This paper has proposed a requirement-first theory of contextual spacetime formation. The guiding idea is simple:
\begin{quote}
There are only requirements. Time and space are part of realization.
\end{quote}
Instead of assuming that cognition takes place inside a fixed spacetime and then asking why probability becomes quantum-like, we began from finite-state requirements: single-state semantics, context-sensitive intervention, avoidance of explicit labels, coherent world-formation, and intersubjective transformability. Under these requirements, a fixed classical spacetime/event form is not always sufficient. When contextual realization is projected back into such a fixed form, the residual appears as interference:
\begin{equation}
P_{\mathfrak{S}^{\ast}\to \mathfrak{S}_{\mathrm{cl}}}(B)=P(A)P(B\mid A)+P(\bar A)P(B\mid \bar A)+I_{\mathfrak{S}^{\ast}}(A,B).
\end{equation}

The central proposal is that \(I_{\mathfrak{S}^{\ast}}(A,B)\) is not merely a probabilistic anomaly. It is the trace of contextual spacetime formation: the holonomy-like mismatch generated when local logic-worlds are glued into a coherent world under finite constraints.

The framework extends quantum cognition by moving from descriptive adequacy to transcendental structure. Quantum cognition is not only a way to model human judgment. It is a window into how finite subjects construct eventhood, objecthood, probability, time, and space.

The framework also extends prior single-state obstruction results.
In prior single-state analyses, classical models preserving fixed state semantics must pay an irreducible contextual information cost. The present theory interprets that cost as the fixed-spacetime shadow of a deeper alternative: instead of storing context labels, a subject may transform the spacetime form itself.

The resulting philosophical position is transcendental-operational realism. It rejects absolute spacetime as primitive, but it also rejects arbitrary subjectivism. Reality is what remains invariant across contextual spacetime formations. Physical objects are stable equivalence classes across such formations. Physical laws are transformation-invariant relations. Facts are high-stability events under intersubjective calibration.

The loss of absolute spacetime therefore does not mean the loss of world. It means that the unique world we inhabit is not primitive. It is a coherent gluing of local logic-worlds:
\begin{equation}
\mathcal{W}=\Glue\bigl(\{L_c\}_{c\in C}\bigr).
\end{equation}
When this gluing is flat, experience appears classical. When it has holonomy, it appears quantum-like. When it resists local spatial decomposition, it appears nonlocal. When it cannot close all its own requirements, it appears incomplete.

The final claim of the paper can therefore be stated as follows:
\begin{quote}
Quantum probability is the fixed-spacetime shadow of contextual realization. Spacetime is not the stage on which requirements are realized; it is one of the forms produced in their realization.
\end{quote}

\section*{Acknowledgments}

This work was supported by SOBIN Institute LLC under Research Grant SP010.  
The author used ChatGPT (OpenAI) for language editing and proofreading. The author takes full responsibility for all conceptual claims, mathematical formulations, references, and the final manuscript.

\end{document}